\documentclass[a4paper, 11pt]{article}

\usepackage{amssymb}
\usepackage{tikz}
\usepackage{authblk}
\usepackage{amsthm}

\newtheorem{definition}{Definition}

\newtheorem{proposition}{Proposition}

\newtheorem{lemma}{Lemma}

\newcommand{\m}{\mathtt}

\newcommand{\cA}{{\mathcal{A}}}

\newcommand{\cE}{{\mathcal{E}}}

\newcommand{\cL}{{\mathcal{L}}}

\newcommand{\cR}{{\mathcal{R}}}

\newcommand{\bA}{{\mathbf{A}}}

\newcommand{\bS}{{\mathbf{S}}}

\newcommand{\la}{\langle}
\newcommand{\ra}{\rangle}
\newcommand{\e}{\emph}

\renewcommand{\deg}{\mathtt{Deg}}
\newcommand{\att}{{\mathtt{Att}}}

\newcommand{\Arg}{{\mathtt{Arg}}}

\newcommand{\ag}{{\mathtt{AG}}}

\newcommand{\ext}{\mathtt{Ext}}

\newcommand{\labx}{\mathtt{Lab}_{x}}

\newcommand{\st}{\mathtt{st}}
\newcommand{\pr}{\mathtt{pr}}
\newcommand{\co}{\mathtt{co}}
\newcommand{\gr}{\mathtt{gr}}
\renewcommand{\i}{{\mathtt{in}}}
\newcommand{\out}{{\mathtt{out}}}
\newcommand{\und}{{\mathtt{und}}}
\title{Dung's semantics satisfy attack removal monotonicity}
\author[1]{Leila Amgoud}
\author[2]{Srdjan Vesic}
\affil[1]{CNRS – IRIT, France (amgoud@irit.fr)}
\affil[2]{ CRIL - CNRS, Univ. Artois, France (vesic@cril.fr)}
\date{}
\begin{document}
\maketitle
\begin{abstract}
We show that preferred, stable, complete and grounded semantics satisfy attack removal monotonicity. This means that if an attack from $b$ to $a$ is removed, the status of $a$ cannot worsen, e.g.\ if $a$ was skeptically accepted, it cannot become rejected. 
\end{abstract}
\section{Introduction}

Formal argumentation theory \cite{Dung95} is non-monotonic in the sense that when new arguments are added, some arguments may change their status. In this rapport, we show that preferred, stable, complete and grounded semantics satisfy attack removal monotonicity. This means that if an attack from $b$ to $a$ is removed, the status of $a$ cannot worsen, e.g.\ if $a$ was skeptically accepted, it cannot become rejected. Note that result we prove in the present document is the proof of Proposition 1 and Conjecture 1 of the recent paper by Amgoud et al.\ \cite{ABNV17IJCAI}.

\section{Formal setting}

Let us introduce the formal setting.
Let $\Arg$ be an infinite set of all possible arguments. 
An argumentation graph is defined as follows:

\begin{definition}[Argumentation Graph]\label{AG}
An \e{argumentation graph} is an ordered tuple $\bA = \la \cA, \cR\ra$, where 
$\cA$ is a finite subset of $\Arg$ and $\cR \subseteq \cA \times \cA$. Let $\ag$ be the universe of 
all argumentation graphs built on $\Arg$.
\end{definition}

If $(a,b) \in \cA$ we say that $a$ attacks $b$. The set of attackers of an argument is defined as $\m{Att}_\bA(a) = \{b \in \cA \mid (b,a) \in \cR \}$. For a set of attacks $X \subseteq \cR$, we define $\bA \circleddash X = \la \cA, \cR \setminus X\ra$.

A set of arguments is \e{conflict free} if it does not contain two arguments $a$ and $b$ such that $(a,b) \in \cR$. A set of arguments $X$ defends itself iff for every argument $x \in X$, for every argument $y$, if $y$ attacks $x$, there exists $z \in X$ such that $z$ attacks $y$. Let $\cE \subseteq \cA$ be a conflict-free set.
\begin{itemize}
\item $\cE$ is a \emph{complete} extension iff it defends all its elements and contains all arguments it defends.
\item $\cE$ is a \emph{preferred} extension iff it is a maximal (w.r.t.\  $\subseteq$) complete extension.
\item $\cE$ is a \emph{stable} extension iff it attacks any $a \in \cA \setminus \cE$.
\item $\cE$ is a \emph{grounded extension} iff it is a minimal (w.r.t.\ set $\subseteq$) complete extension.
\end{itemize}

Let $\ext_x(\bA)$ denote the set of all extensions of $\bA$ under semantics $x$, where 
$x \in \{\st, \pr, \gr, \co\}$ and $\st$ (resp. $\pr, \gr, \co$) stands for 
stable (resp. preferred, grounded, complete) semantics. Once extensions are computed, 
acceptability degrees are assigned to arguments. We follow the definition by Amgoud et al.\ \cite{ABNV17IJCAI}. 

\begin{definition}
\label{degree1}
For every $a \in \cA$,  
if $\ext_x(\bA) = \emptyset$, then $\mathtt{Deg}^{\bA}_{x}(a) = 1$, otherwise, $\mathtt{Deg}^{\bA}_{x}(a)=$
\begin{itemize}
	\item  $1$ iff $a \in \bigcap \limits_{\cE \in \ext_x(\bA)} \cE$.

	\item $0.5$ iff $\exists \cE, \cE'  \in \ext_x(\bA)$ s.t $a \in \cE$, $a \notin \cE'$.
	
	\item $0.3$ iff $a \notin \bigcup \limits_{\cE \in \ext_x(\bA)} \cE$ and 
	$\nexists \cE \in \ext_x(\bA)$ s.t $\exists b \in \cE$ and $(b,a) \in \cR$.		      
	
	\item $0$ iff $a \notin \bigcup \limits_{\cE \in \ext_x(\bA)} \cE$ and 
	$\exists \cE \in \ext_x(\bA)$ s.t $\exists b \in \cE$ and $(b,a) \in \cR$. 
\end{itemize}
\end{definition}

\section{Attack removal monotonicity}

The principle we study, i.e.\ attack removal monotonicity, was introduced by Amgoud et al.\ \cite{ABNV17IJCAI}, in Definition 6. The principle was simply called \textit{monotonicity}. We prefer to use a more specific name, i.e.\ attack removal monotonicity, in order to distinguish this principle from other principles in the literature that are also called monotonicity, e.g.\ \cite{AmgoudB16b}. 

\begin{definition}[Attack removal monotonicity]
A semantics $\bS$ satisfies attack removal monotonicity iff for every argumentation graph $\bA = \la \cA, \cR \ra \in \ag$, 
$\forall a \in \cA$, $\forall X \subseteq \att_{\bA}(a)$, it holds that
$\mathtt{Deg}^{\bA}_{\bS}(a) \leq \mathtt{Deg}^{\bA \circleddash X}_{\bS}(a)$. 
\end{definition}

\section{Labellings}

To prove that preferred, stable, complete and grounded semantics satisfy this principle, we first introduce the basic concepts behind the labeling \cite{caminada}.

\begin{definition}[Labeling]\label{labelling}
	A \e{labeling} is a total function $\cL$ which assigns to each argument  
	of every argumentation graph $\bA = \la \cA, \cR \ra$ a value in the set 
	$\{\i, \out, \und\}$. A labeling $\cL$ is \e{complete} iff 
	$\forall a \in \cA$, 
	\begin{itemize}
		\item $\cL(a) = \i$ if $\forall b \in \mathtt{Att}_{\bA}(a)$, $\cL(b) = \out$,
		\item $\cL(a) = \out$ if $\exists b \in \mathtt{Att}_{\bA}(a)$ such that $\cL(b) = \i$,
		\item $\cL(a) = \und$ otherwise. 
	\end{itemize} 
	Let $\mathtt{Lab_c}(\bA)$ be the set of all complete labellings of $\bA$.
\end{definition}

Let $\bA = \la \cA, \cR \ra$ be an argumentation graph. For a set of arguments $X \subseteq \cA$, we define the corresponding labelling $\mathtt{Ext2Lab}(X)$ as follows:
\begin{itemize}
	\item $\forall x \in X$, $\cL(x) = \i$,
	\item $\forall y \in \cA \setminus X$ such that $\exists x \in X$ and $x \cR y$, $\cL(y) = \out$, 
	\item $\forall y \in \cA \setminus X$ such that $\nexists x \in X$ and $x \cR y$, $\cL(y) = \und$. 
\end{itemize} 

For a labeling $\cL$ of $\bA$, we define $\mathtt{Lab2Ext}(\cL) = \{x \in \cA \ | \ \cL(x) = \i\}$. 

It was shown by Caminada \cite{caminada} that the set of arguments labeled $\i$ within a complete labeling 
forms a complete extension of the argumentation graph. Similarly, every complete extension can 
be labeled by a complete labeling. Correspondences between complete labellings and the three 
other semantics (grounded, preferred, stable) have also been shown. We denote by $\mathtt{Lab_s}(\bA)$ (respectively $\mathtt{Lab_p}(\bA)$, 
$\mathtt{Lab_g}(\bA)$ ) the set of labellings corresponding to stable  (respectively preferred, grounded) extension(s) of $\bA$. 

\section{Two lemmas}
\begin{lemma}\label{addition}
	Let $\bA = \la \cA, \cR \ra \in \ag$. 
	Let $a \in \cA$ be such that $\att_{\bA}(a) \neq \emptyset$, and $j \in \{\st, \pr, \co, \gr\}$. 
	For any $\emptyset \subseteq X \subseteq \att_{\bA}(a)$, 
	for any $\cL \in \mathtt{Lab_j}(\bA \circleddash X)$, if $\cL(a) = \out$, 
	then $\cL \in \mathtt{Lab_j}(\bA)$.
\end{lemma}

\begin{proof}
	Let $\bA = \la \cA, \cR \ra \in \ag$, and 
	$j \in \{\st, \pr, \co, \gr\}$. Let $a \in \cA$ be such that $\att_{\bA}(a) \neq \emptyset$, 
	and let $\emptyset \subseteq X \subseteq \att_{\bA}(a)$. Let $X = \{x_1, \ldots, x_n\}$ and 
	$\bA' = \bA \circleddash X$. 
	Assume that $\cL \in \mathtt{Lab}_j(\bA')$ and $\cL(a) = \out$. 
	From Table 1 (in \cite{leon}), $\cL \in \mathtt{Lab_j}(\bA_1)$ where 
	$\bA_1 = \bA' \oplus x_1$\footnote{For $\bA = \la\cA, w, \cR\ra$, we denote by 
		$\bA \oplus x$ the argumentation graph $\bA = \la\cA, w, \cR\cup\{x\}\ra$.}. 
	We construct thus a series of argumentation graphs $\bA_1, \ldots, \bA_n$, 
	where for each $i = 2, \ldots, n$, $\bA_i = \bA_{i-1} \oplus x_i$, and we 
	repeat the previous reasoning in order to obtain that $\cL \in \mathtt{Lab_j}(\bA_i)$.  
	Note that $\bA_n = \bA$. Hence, $\cL \in \mathtt{Lab_j}(\bA)$.
\end{proof}

\begin{lemma}\label{removal}
	Let $\bA = \la \cA, \cR \ra \in \ag$.  
	Let $a \in \cA$ be such that $\att_{\bA}(a) \neq \emptyset$, and $j \in \{\st, \pr, \co, \gr\}$. 
	For any $\emptyset \subseteq X \subseteq \att_{\bA}(a)$, for any $\cL \in \mathtt{Lab_j}(\bA)$, 
	if $\cL(a) = \i$, then $\cL \in \mathtt{Lab_j}(\bA\circleddash X)$.
\end{lemma}
\begin{proof}
	Let $\bA = \la \cA, \cR \ra \in \ag$ and 
	$j \in \{\st, \pr, \co, \gr\}$. Let $a \in \cA$ be such that $\att_{\bA}(a) \neq \emptyset$, 
	$\emptyset \subseteq X \subseteq \att_{\bA}(a)$, $X = \{x_1, \ldots, x_n\}$, and $\bA' = \bA \circleddash X$. 
	Assume that $\cL \in \mathtt{Lab_j}(\bA)$ and $\cL(a) = \i$. 
	From Table 2 (in \cite{leon}), $\cL \in \mathtt{Lab_j}(\bA_1)$ where 
	$\bA_1 = \bA \circleddash \{x_1\}$. 
	We construct thus a series of argumentation graphs $\bA_1, \ldots, \bA_n$, 
	where for each $i = 2, \ldots, n$, $\bA_i = \bA_{i-1} \circleddash \{x_i\}$, and we 
	repeat the previous reasoning in order to obtain that $\cL \in \mathtt{Lab_j}(\bA_i)$.  
	Note that $\bA_n = \bA \circleddash X$. Hence, $\cL \in \mathtt{Lab_j}(\bA\circleddash X)$.
\end{proof}

\section{The main result}
\begin{proposition}
\label{extension-semantics-monotone}
Stable, preferred, complete and grounded semantics satisfy attack removal monotonicity. 
\end{proposition}
\begin{proof}
Let $\bA = \la \cA, \cR \ra \in \ag$. 
	Let $a \in \cA$ be such that $\att_{\bA}(a) \neq \emptyset$, 
	$X \subseteq \att_{\bA}(a)$, and 
	$x \in \{\co, \st, \pr, \gr\}$. 
	We put $\bA' = \bA\circleddash X = \langle \cA',\cR'\rangle$, where $\cR' = \cR \setminus X$. 
	Let us show that 
	$\deg^{\bA'}_{x}(a) \geq \deg^{\bA}_{x}(a)$.

	\begin{description}
		\item[Case 1.] $X = \att_{\bA}(a)$. This means that $\att_{\bA'}(a) = \emptyset$. 
		\begin{itemize}
			\item $\mathtt{Ext_x}(\bA') = \emptyset$ (the case is possible only for stable semantics). By definition $\deg^{\bA'}_x(a) = 1$. 
			
			\item $\mathtt{Ext_x}(\bA') \neq \emptyset$. From \cite{Dung95}, 
			$a \in \bigcap\limits_{\cE \in \mathtt{Ext_x}(\bA')} \cE$, and 
			hence $\deg^{\bA'}_x(a) = 1$.      	     
		\end{itemize}
		Since by definition, 
		$\deg^{\bA}_x(a) \in \{0, 0.3, 0.5, 1\}$, then $\deg^{\bA}_x(a) \leq 1$.
		\item[Case 2.] $X = \emptyset$. Hence, $\bA$ and $\bA'$ coincide, trivial. 
		%
		\item[Case 3.] $\emptyset \subsetneq X \subsetneq \att_{\bA}(a)$. 
		
		i) Assume that $\deg^{\bA'}_{x}(a) = 1$. Since by definition, 
		$\deg^{\bA}_x(a) \in \{0, 0.3, 0.5, 1\}$, then $\deg^{\bA}_x(a) \leq 1$. 
		
		ii) Assume now that $\deg^{\bA'}_{x}(a) < 1$. Thus, by definition $\ext_x(\bA') \neq \emptyset$. From now on, we proceed per semantics. 
		
		\begin{description}
		\item [Stable semantics:] $\deg^{\bA'}_{\st}(a) \in \{0, 0.5\}$
			\begin{itemize}
				\item Assume that $\deg^{\bA'}_{\st}(a) = 0.5$. 
				By definition, $\exists \cE \in \mathtt{Ext_\st}(\bA')$ such that $a \notin \cE$. 
				Let $\cL = \mathtt{Ext2Lab}(\cE)$. Since with stable semantics, $\cL(.) \in \{\i,\out\}$, then $\cL(a) = \out$.
				From Lemma \ref{addition}, $\cL \in \mathtt{Lab_\st}(\bA)$. 
				Consequently, $\deg^{\bA}_{\st}(a) < 1$.  
				
				\item Assume that $\deg^{\bA'}_{\st}(a) = 0$. By definition, 
				$\exists \cE \in \ext_\st(\bA')$ such that an element $b$ of $\cE$ attacks $a$ (i.e. $b \cR' a$). 
				Let $\cL = \mathtt{Ext2Lab}(\cE)$.
				From \cite{caminada}, $\cL \in \mathtt{Lab}_\st(\bA')$ and    
				$\cL(a) = \out$.  Furthermore, 
				$\nexists \cL' \in \mathtt{Lab}_\st(\bA')$ such that
				\begin{equation}
				\label{1}
				    				\cL'(a) = \i 
				\end{equation}
There are two cases:
				\begin{itemize}
					\item $\mathtt{Ext_\st}(\bA) = \emptyset$. But from Lemma \ref{addition}, $\cL \in  \mathtt{Lab_\st}(\bA)$ (meaning that $\cE$ is a stable extension of $\bA$). Contradiction.
					\item $\mathtt{Ext_\st}(\bA) \neq \emptyset$. Assume that 
					$\deg^{\bA}_{\st}(a) > 0$. Thus, 
					$\deg^{\bA}_{\st}(a) \in \{0.5, 1\}$ (under stable semantics, an argument cannot have degree 0.3). Hence,
					$\exists \cL' \in \mathtt{Lab}_\st(\bA)$ 
					such that $\cL'(a) = \i$. From Lemma \ref{removal}, $\cL' \in \mathtt{Lab_\st}(\bA')$. 
					This contradicts assumption $(1)$.
				\end{itemize}	
			\end{itemize}
			
			\item [Grounded semantics:] $\deg^{\bA'}_{\st}(a) \in \{0, 0.3\}$
			
			Let $\cL  = \mathtt{Ext2Lab}(\mathtt{GE}(\bA))$ 
			and $\cL' = \mathtt{Ext2Lab}(\mathtt{GE}(\bA'))$.
			\begin{itemize}
				\item Assume that $\deg^{\bA'}_{\gr}(a) = 0.3$. 		       
				By definition, $(\mathtt{Atts}_{\bA'}(a) \cup \{a\}) \cap \mathtt{GE}(\bA') = \emptyset$, 
				and from \cite{caminada}, $\cL'(a) = \und$. 
				
				Assume that $\deg^{\bA}_{\gr}(a) > 0.3$. Hence, 
				$\deg^{\bA}_{\gr}(a) = 1$ (since $\deg^{\bA}_{\gr}(a) \in \{0, 0.3, 1\}$). 
				From \cite{caminada}, $\cL(a) = \i$. From Lemma \ref{removal}, 
				$\cL = \cL'$, hence contradiction. 		      		
				
				\item Assume that $\deg^{\bA'}_{\gr}(a) = 0$. 
				From \cite{caminada}, $\cL'(a) = \out$. But from Lemma \ref{addition}, 
				$\cL' = \cL$. Hence, $\cL(a) = \out$ and $\deg^{\bA}_{\gr}(a) = 0$.	
			\end{itemize}

			\item [Complete and preferred semantics:] $\deg^{\bA'}_{x}(a) \in \{0, 0.3, 0.5\}$ with $x \in \{\co, \pr\}$.
			\begin{itemize}	
				\item Assume that $\deg^{\bA'}_{x}(a) = 0.5$. By definition, there exists $\cL \in \mathtt{Lab}_x(\bA')$ 
				such that $\cL(a) = \out$. From Lemma \ref{addition}, 
				$\cL \in \mathtt{Lab}_x(\bA)$. Then, $\deg^{\bA}_{x}(a) < 1$.
				

				\item Assume that $\deg^{\bA'}_{x}(a) = 0.3$. 		       
				By definition, $a$ does not belong to any complete/preferred extension. 
				Hence, $\nexists \cL \in \labx(\bA')$ such that $\cL(a) = \i$ (\textbf{1}).
				Assume that $\deg^{\bA}_{x}(a) > 0.3$. So, 
				$\deg^{\bA}_{x}(a) \in \{0.5, 1\}$). From \cite{caminada}, 
				$\exists \cL \in \labx(\bA)$ such that $\cL(a) = \i$. From Lemma \ref{removal}, 
				$\cL \in \labx(\bA')$, contradiction with \textbf{1}. 		      		
				
				\item Assume that $\deg^{\bA'}_{x}(a) = 0$. By definition of $\deg^{\bA'}_{x}(a)$, 
				$a$ is attacked by at least one complete/preferred extension $\cE$. 
				Thus, for $\cL = \mathtt{Ext2Lab}(\cE)$, it holds that $\cL(a) = \out$. 
				From Lemma \ref{addition}, $\cL \in \labx(\bA)$. Assume now that $\deg^{\bA}_{x}(a) > 0$.
				There are three cases:
				\begin{itemize}
					\item $\deg^{\bA}_{x}(a) = 1$. This is impossible since $\forall \cL' \in \labx(\bA)$, $\cL'(a) = \i$.
					\item $\deg^{\bA}_{x}(a) = 0.5$. This means that $\exists \cL' \in \labx(\bA)$ 
					such that $\cL'(a) = \i$. From Lemma \ref{removal}, $\cL' \in \labx(\bA')$. This would mean 
					that $\deg^{\bA'}_{x}(a) \geq 0.5$. This contradicts the assumption 
					$\deg^{\bA'}_{x}(a) = 0$. 
					\item $\deg^{\bA}_{x}(a) = 0.3$. This is impossible since $\cL \in \labx(\bA)$ 
					and $\cL(a) = \out$ meaning that $a$ is attacked by at least one complete/preferred extension.
				\end{itemize}	
			\end{itemize}
		\end{description}
	\end{description}
\end{proof}

\section{Appendix: the proof of the same result under an alternative definition of degrees}

The choice of Definition \ref{degree1} to assign the degree $1$ to all the arguments in case when there are no extensions can be disputable. We show that our main result still holds under an alternative definition of degrees. 

Let us present an alternative definition of degrees, where all the arguments are assigned the degree $0$ in case there are no extensions. The rest of the following definition coincides with Definition \ref{degree1}.

\begin{definition}
\label{degree2}
For every $a \in \cA$,  
if $\ext_x(\bA) = \emptyset$, then $\mathtt{Deg}^{\bA}_{x}(a) = 0$, otherwise, $\mathtt{Deg}^{\bA}_{x}(a)=$
\begin{itemize}
	\item  $1$ iff $a \in \bigcap \limits_{\cE \in \ext_x(\bA)} \cE$.

	\item $0.5$ iff $\exists \cE, \cE'  \in \ext_x(\bA)$ s.t $a \in \cE$, $a \notin \cE'$.
	
	\item $0.3$ iff $a \notin \bigcup \limits_{\cE \in \ext_x(\bA)} \cE$ and 
	$\nexists \cE \in \ext_x(\bA)$ s.t $\exists b \in \cE$ and $(b,a) \in \cR$.		      
	
	\item $0$ iff $a \notin \bigcup \limits_{\cE \in \ext_x(\bA)} \cE$ and 
	$\exists \cE \in \ext_x(\bA)$ s.t $\exists b \in \cE$ and $(b,a) \in \cR$. 
\end{itemize}
\end{definition}

\begin{proposition}
\label{extension-semantics-monotone-degrees2}
Stable, preferred, complete and grounded semantics satisfy attack removal monotonicity under the alternative definition of acceptability degree (Definition \ref{degree2}).
\end{proposition}
\begin{proof}
We only need to consider stable semantics, since the proof for the other semantics is identical to the proof of Proposition \ref{extension-semantics-monotone}. 
Let $\bA = \la \cA, \cR \ra \in \ag$. 
	Let $a \in \cA$ be such that $\att_{\bA}(a) \neq \emptyset$, 
	let $X \subseteq \att_{\bA}(a)$. 
	We define $\bA' = \bA\circleddash X = \langle \cA',\cR'\rangle$, where $\cR' = \cR \setminus X$. 
	Let us show that 
	$\deg^{\bA'}_{\st}(a) \geq \deg^{\bA}_{\st}(a)$.
	\begin{description}
	\item[Case 1.] If $\deg^{\bA'}_{\st}(a) = 1$, the proof is over. 
	\item[Case 2.] Let  $\deg^{\bA'}_{\st}(a) = 0.5$. Let $\cL$ be a stable labelling of $\bA'$ such that $\cL(a) = \out$.  From Lemma \ref{addition}, $\cL$ is a stable labelling of $\bA$. Hence $\deg^{\bA}_{\st}(a) < 1$. 
	
	\item[Case 3.] Let $\deg^{\bA'}_{\st}(a) = 0$. It might be that there are no stable extensions of $\bA'$ or that there is a stable extension of $\bA'$. In both cases, we proceed by the proof by reductio ad absurdum. Aiming at the contradiction, let us suppose $\deg^{\bA}_{\st}(a) > 0$. This means that $\bA$ has at least one stable extension. Let $\cL$ be a stable labelling of $\bA$ such that $\cL(a) = \i$. From Lemma \ref{removal}, $\cL$ is a stable labelling of $\bA'$. Hence, $a$ belongs to at least one extension of $\bA'$. Contradiction.
	\end{description}
\end{proof}

\bibliographystyle{plain}
\bibliography{general}
\end{document}